 \documentclass[pmlr,twocolumn,10pt]{jmlr} % W&CP article

\usepackage{booktabs}
\usepackage{siunitx}

\usepackage[switch]{lineno}

\theorembodyfont{\upshape}
\theoremheaderfont{\scshape}
\theorempostheader{:}
\theoremsep{\newline}

\jmlrvolume{259}
\jmlryear{2024}
\jmlrsubmitted{LEAVE UNSET}
\jmlrpublished{LEAVE UNSET}
\jmlrworkshop{Machine Learning for Health (ML4H) 2024} % W&CP title

\title[Training-Aware Risk Control for IMRT Quality Assurance with Conformal Prediction]{Training-Aware Risk Control for Intensity Modulated Radiation Therapies Quality Assurance with Conformal Prediction}

\author{%
\Name{Kevin He} \Email{khe7@jhu.edu}\\
\addr Department of Computer Science, Johns Hopkins University
\AND
\Name{David Adam} \Email{dadam3@jh.edu}\\
\addr Department of Radiation Oncology, Johns Hopkins School of Medicine
\AND
\Name{Sarah Han-Oh} \Email{yhanoh1@jhmi.edu}\\
\addr Department of Radiation Oncology, Johns Hopkins School of Medicine
\AND
\Name{Anqi Liu} \Email{aliu74@jhu.edu}\\
\addr Department of Computer Science, Johns Hopkins University
}

\begin{document}

\maketitle

\begin{abstract}

Measurement quality assurance (QA) practices play a key role in the safe use of Intensity Modulated Radiation Therapies (IMRT) for cancer treatment. These practices have reduced measurement-based IMRT QA failure below 1\%. However, these practices are time and labor intensive which can lead to delays in patient care. In this study, we examine how conformal prediction methodologies can be used to robustly triage plans. We propose a new training-aware conformal risk control method by combining the benefit of conformal risk control and conformal training. 
% Our best models with conformal prediction methods can achieve sensitivities of 1 and specificity between 0.59-0.86 depending on the hospital dataset, outperforming all-known machine learning models 
We incorporate the decision-making thresholds based on the gamma passing rate, along with the risk functions used in clinical evaluation, into the design of the risk control framework.
Our method achieves high sensitivity and specificity and significantly reduces the number of plans needing measurement without generating a huge confidence interval. Our results demonstrate the validity and applicability of conformal prediction methods for improving efficiency and reducing the workload of
the IMRT QA process.
% By focusing on maintaining a high sensitivity and ensuring robust predictions, we hope to contribute to the adoption of AI in IMRT QA.

\end{abstract}

\paragraph*{Data and Code Availability}
In this project, we used 2 IMRT plan datasets collected during 2 time periods at the same site on the same machine at Johns Hopkins Hospital. The first dataset was from Johns Hopkins Hospital (JHH1) with cases from 6/23 to 8/23. The second dataset was from Johns Hopkins Hospital (JHH2) with cases from 9/23 to 12/23. The data is not publicly available but the code is available at: https://github.com/khe9370/Training-Aware-CRC
% The third dataset was from Johns Hopkins Green Spring Station (G2). The fourth dataset was from Johns Hopkins Sibley Memorial Hospital (Sibley).

% Data and code is available in the supplement.

% \paragraph*{Institutional Review Board (IRB)}
% IRB information will be provided upon acceptance.

\section{Introduction}
\label{sec:intro}

Intensity Modulated Radiation Therapy (IMRT) is a form of cancer treatment that delivers a precise dose of radiation to a tumor while sparing the surrounding tissue \citep{Jaffray}. As an important part of the treatment process, IMRT plans for patients undergo measurement quality assurance (QA) to ensure accurate delivery of radiation \citep{AAPM-Task-Force-119, AAPM-Task-Force-218}. The most common way to measure quality is based on a gamma passing rate (GPR), which is the percentage of points on a dosimeter that meet a dose-difference and spatial-difference criteria \citep{Valdes,Interian,Lam}. If the GPR is inaccurately predicted and an unsafe IMRT plan is used to treat a patient, the adverse effects could include either overdosing or underdosing the tumor compared to the prescribed dose, and a potential overdose of healthy organs near the tumor \citep{Palta}. Overdosing healthy organs can result in radiation-induced side effects, which the patient may suffer from in the short or long term, depending on the type and severity of the side effects \citep{Eaton}. Underdosing the tumor reduces the tumor control probability, thereby decreasing the efficacy of radiotherapy \citep{Eaton}. Current American Association of Physicists in Medicine (AAPM) recommendations set universal tolerance of radiation limits at 95\% passing rate using a 3\% dose difference/2 mm spatial difference (3\%/2mm) criteria \citep{AAPM-Task-Force-218}. However, using these passing rates with 3\% dose difference/3mm spatial difference (3\%/3mm) criteria is still widely accepted \citep{AAPM-Task-Force-218, Interian}. Figure \ref{fig: workflow} shows a typical process of delivering IMRT treatment in a radiation oncology department, where QA plays a critical role. 

While IMRT QA practices significantly improve patient outcomes, it comes at the cost of being a resource-intensive and time-consuming process \citep{Valdes, Palta, Lam}. Medical physicists have to make iterative deliveries of the treatment plan to a dosimeter, sometimes overnight or on weekends, to make the adjustments necessary so the plan is deemed satisfactory \citep{Smilowitz}. In response to this challenge, hospitals have tried to develop machine learning models to triage plans that need to undergo measurement \citep{Valdes,Interian,Lam}. These models have shown the ability to make relatively accurate GPR predictions, but are difficult to deploy in practice since they struggle with classifying plan safety based upon universal tolerance limits \citep{Chan}. 
\begin{figure*}[h]
    \centering
    \includegraphics[width=0.99\linewidth]{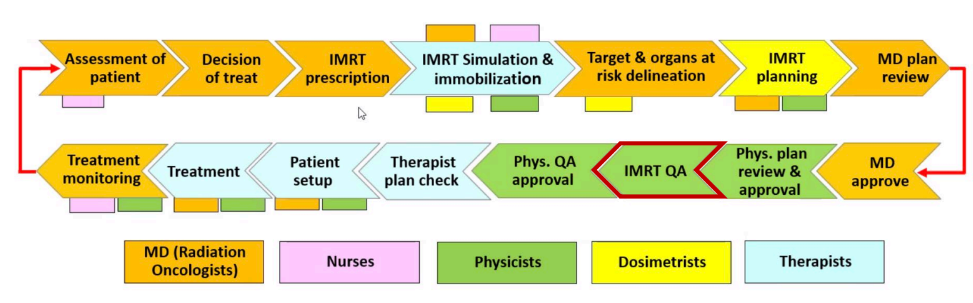}
    \caption{The workflow of a radiation oncology department for delivering IMRT treatment. The IMRT process involves multiple different disciplines and professionals in the workflow to design and implement a personalized treatment plan based on a patient's unique disease condition. The IMRT QA process is the evaluation process after a plan is designed and before a plan is deployed in the treatment. It is a safety-critical task as we do not want to deliver low-quality treatment to patients. Overdosing healthy organs can result in radiation-induced side effects, while underdosing the tumor reduces the tumor control probability, thereby decreasing the efficacy of radiotherapy.}
    \label{fig: workflow}
\end{figure*}

One paradigm that can help solve these problems is conformal prediction. Conformal prediction is a distribution-free uncertainty quantification framework that quantifies uncertainty by producing a statistically valid prediction region \citep{Angelopoulos-1}, meaning these regions will contain the true label with high probability.
By producing such a prediction region, we can triage measurement plans with significant ranges that extend into dangerous GPRs. Recent work has also explored the idea of controlling the risk of the prediction with guarantees \citep{Angelopoulos-2}. However, utilizing conformal prediction methodologies for improving clinical decision making in IMRT QA is still an open question.

In our work, we provide a conformal prediction based solution to reduce the number of plans that need to undergo IMRT QA while trying to ensure that no failing plan gets passed through by our model. We employ conformal training methods with a conformal risk control penalty for making predictions and generating triage decisions based upon a controlled risk.

\textbf{Our Key Contributions:}
\begin{enumerate}
     \item  We propose a training-aware conformal risk control method for quality measurement of IMRT treatment plans that considers a penalty from conformal risk control in the model training process. Our method incorporates the actual decision-making thresholding on the GPR and the risk functions used in clinical evaluation into the design of the risk control framework.
     
    \item We compare our proposed method with various baselines in the conformal prediction framework, including standard split conformal \citep{Angelopoulos-1}, conformal quantile regression \citep{Romano}, conformal risk control \citep{Angelopoulos-2}, and conformal training methods \citep{Stutz}. Evaluating on 2 real-world IMRT treatment plan datasets, our method achieves high sensitivity with better specificity than all the baselines, without generating a huge interval.

    \item Through our analysis, we demonstrate the applicability of conformal prediction methods in improving efficiency and reducing the workload of the IMRT QA process when the machine learning design choices are carefully made to match with clinical practices.

    % Johns Hopkins Health System can generalize well to evaluating other IMRT data by evaluating IMRT QA performance on IMRT data from Johns Hopkins' Health Care and Surgery Center - Green Spring Station and Sibley Memorial Hospital.

\end{enumerate}

\section{Related Work}

\subsection{IMRT QA Prediction with Machine Learning Models}
Several previous works have developed machine learning models to assess the quality of measurement plans. Most models in this field are trained on a combination of bi-dimensional patient-specific quality assurance results, plan complexity metrics, and linear accelerator performance metrics to predict the gamma passing rate as part of virtual patient-specific quality assurance. These models provide a point prediction for the GPR and their mean absolute error is relatively tight \citep{simon-robert-meyer}. \citet{Lam} trained a Random Forest Regressor, an adaBoost Regressor, and XGBoost Regressor to predict GPRs. They found that all models had mean absolute errors to be within 0.9-1\%. Similarly, \citet{Interian} developed a convolutional neural network model to perform this task and they had a mean absolute error of 0.74\%. 

However, mean absolute error does not directly translate to informative results for clinical decision-making. Medical physicists assess plan safety based upon a plan meeting a gamma pass rate threshold \citep{AAPM-Task-Force-119, AAPM-Task-Force-218}, as well as taking into account other information including the specific treatment region and so on. Virtual IMRT QA needs to be assessed based upon sensitivity and specificity. Almost all groups use at least a 90\% GPR threshold and many groups use a 95\% GPR threshold for a 3\%/3mm or 3\%/2mm criteria \citep{AAPM-Task-Force-218}. High sensitivity is required to ensure that no unsafe plan gets passed through. 

When we evaluate the models from \citep{Lam, Interian} by sensitivity and specificity on a 95\% gamma passing threshold with a 3\%/3mm criteria, they achieve the following sensitivity and specificity: Random Forest Regressor achieves 0.5625 in sensitivity and 0.92 in specificity;  AdaBoost Regressor achieves 0.688 in sensitivity and 0.92 in specificity; XGBoost Regressor achieves 0.688 in sensitivity and 0.91 in specificity; Convolutional Neural Networks achieve 0.622 in sensitivity and 0.995 in specificity \citep{AAPM-Task-Force-218, Interian, Lam}. These low sensitivities pose significant risks for patient safety and limit the adoption of machine learning models for IMRT QA.

\subsection{Conformal Prediction in the Healthcare Domain}
While conformal prediction methods have never been applied to improving the IMRT QA process, conformal prediction has been gaining traction in other healthcare applications due to its ability to provide uncertainty estimates alongside predictions. Most of these studies use conformal prediction to improve the quality of predictions on novel healthcare datasets \citep{Vasquez}. As a few examples, \cite{Pereira} use support vector machines (SVM) and applied standard inductive conformal prediction to make predictions on the progression of mild cognitive impairment to dementia. \citet{Papadopoulos} use a Multi-Layered Perceptron (MLP) with conformal prediction to make predictions for the diagnosis of abdominal pain based upon 33 associated abdominal pain symptoms. While these papers generate valuable clinical insights, they do not assess the quality of using different variations of conformal prediction methods for the problem. They also do not discuss how clinicians can apply uncertainty estimations to mitigate risk.

Within conformal prediction methodology research, researchers have also occasionally applied their new conformal prediction methods to healthcare applications. In the original conformal risk control paper \citep{Angelopoulos-2}, the authors apply their method to segmenting cancerous polyps and controlling the false negative rate risk \citep{Angelopoulos-2}. \cite{Argaw} study the heterogeneous effects of randomized clinical trials and make predictions using joint confidence intervals derived from conformal quantile regression. These works suggest that achieving bounded risk, based on a specific risk function, would be more suitable than a standard conformal prediction that only produces a coverage guarantee for the prediction sets in risk-sensitive settings. However, despite rigorous guarantees, conformal risk control remains a post hoc process, whose empirical performance is affected significantly by the model trained before the risk-controlling process.

In this work, we propose a training-aware conformal risk control method for predicting the GPR of IMRT treatment plans that considers a penalty from conformal risk control in the model training process. To better align the machine learning solution design and the empirical clinical practices, we also incorporate the actual decision-making threshold of the GPR and the risk functions used in clinical evaluation into the design of the risk control framework.

\section{Methodology}
We begin by defining the general structure of conformal prediction and then we describe conformal quantile regression, conformal risk control, and conformal training method. Finally, we explain the proposed method built off of conformal training and conformal risk control.

\subsection{Conformal Prediction}
In our general setup, we aim to predict a 3\%/3 mm GPR $Y_{i}$, bounded between 0 and 100, based upon its associated IMRT plan features called $X_{i}$ where i ranges from 1 to $n$, the total number of calibration samples. 
In conformal prediction, the main difference is that instead of generating a point estimate of GPR $Y_{i}$, we aim to generate a prediction set or interval that covers the true label $Y_{i}$ with high probability. The method is distribution-free as it does not require any assumptions on the data distribution. The method is also model-agnostic as it (the split conformal prediction version) is usually a posthoc process that can work with any kind of trained prediction models. So we can train a regression model to predict the GPR $Y_{i}$ first. Then we need to define a coverage value $\alpha$ that is our theoretical miscoverage rate such that $\mathbb{P}[Y\in\hat{C}] \geq 1 - \alpha$, where $\hat{C}$ is the prediction set generated in conformal prediction. In standard split conformal prediction, we split our feature dataset $X$ into three portions: $X_{train}$, $X_{val}$, and $X_{test}$. We split our 3\%/3mm GPR results into their corresponding $y_{train}$, $y_{val}$, and $y_{test}$. We utilize $X_{train}$ and $y_{train}$ to train our model and then use this model to make a set of predictions $\hat{y}_{val} $ based on $X_{val}$. We define non-conformity scores as $|\hat{y}_{val} - y_{val}|$, which is the absolute value of the error, i.e., the difference between the predicted GPR and the true GPR for each validation data point. Our one-side width of the confidence interval $I$ will be the $(1 - \alpha)$th-percentile largest non-conformity value in the validation data \citep{Vovk}. For our final confidence interval for evaluation $\hat{C}$, we generate a set of predictions of $\hat{y}_{test}$ from $X_{test}$ and create the confidence interval:
\begin{equation}
    \hat{C} = [\hat{y}_{test} - I, \hat{y}_{test} + I].
\end{equation}
Theoretically, if the calibration data and the testing data are exchangeable (meaning we can swap around, or reorder, variables in the data sequence without changing their joint distribution), then we have the following coverage guarantee, for a test point $X_{n+1}$:
\begin{equation}
    \mathbb{P}(Y_{n+1}\notin \hat{C}(X_{n+1})) \leq \alpha. \label{eq:coverage}
\end{equation}
In our application, $\alpha$ is a user-defined parameter that will be designed in collaboration with physicians, more specifically radiation oncologists.

\subsection{Conformal Quantile Regression}
Conformal quantile regression is an alternative method for generating confidence intervals. Under conformal quantile regression, we are able to develop dynamic confidence intervals based upon a low and high quantile regression model. We choose an $\alpha$ value to represent our theoretical miscoverage and use $X_{train}$ and $y_{train}$ to train one quantile regression model as $q_{low}$ to predict $\frac{\alpha}{2}$th-percentile of the data and one quantile regression model as $q_{high}$ to predict the $1 - \frac{\alpha}{2}$th-percentile of the data. Then, we take each model and use it to make predictions on $X_{val}$, obtaining $\hat{y}^{low}_{val}$ which are the predictions made with $q_{low}$ and $\hat{y}^{high}_{val}$ which are the predictions made with $q_{high}$. We calculate non-conformity scores for each set defined as $|\hat{y}^{low}_{val} - y_{val}|$ and $|\hat{y}^{high}_{val} - y_{val}|$ and obtain the largest $(1 - \alpha)$-percentile from each set and define these as $I_{low}$ and $I_{high}$ respectively.
Finally, we use $q_{low}$ and $q_{high}$ to make predictions on $X_{test}$ to obtain $\hat{y}^{low}_{test}$ and $\hat{y}^{high}_{test}$. The resulting confidence interval is:
\begin{equation}
    \hat{C} = [\hat{y}^{low}_{test} - I_{low}, \hat{y}^{high}_{test} + I_{high}].
\end{equation}

Conformal quantile regression achieves the same coverage guarantee as Eq. (\ref{eq:coverage}), but it is a ``conditional'' \footnote{Note the true conditional coverage guarantee is not achievable without further assumptions, here the variation in inputs is achieved by quantile regression.} one, meaning the confidence interval varies with different test data points.

\subsection{Conformal Risk Control}
Conformal risk control is a way to achieve a bounded risk utilizing the uncertainty estimated in the prediction sets. Unlike conformal prediction and conformal quantile regression which try to guarantee that 
% \begin{equation}
    $\mathbb{P}(Y_{n+1}\notin \hat{C}(X_{n+1})) \leq \alpha$,
% \end{equation}
where $(X_{n+1}, Y_{n+1})$ is a new test point, $\alpha$ is the theoretical miscoverage rate and $\hat{C}$ is the confidence interval, conformal risk control is designed to control a risk function such that 
\begin{equation}
\mathbb{E} \left[ \ell (\hat{C}_{\lambda}(X_{n+1}), Y_{n+1}) \right] \leq \alpha,
\end{equation}
for any bounded loss function \(\ell\) that decreases as the size of \(\hat{C}_{\lambda}(X_{n+1})\) grows \citep{Angelopoulos-2}, where $\lambda$ is a parameter that controls the width of the prediction interval. 

In our model, the risk that we want to control is the sensitivity, the percentage of plans below 95\% GPR that are correctly classified, as we do not want to pass through plans that have below 95\% GPR for a 3\%/3mm criteria.  So we penalize the prediction intervals that are completely above the 95\% threshold when the actual passing rate is below 95\%. We define the risk as follows:
\begin{align}
&\ell (\hat{C}_{\lambda}(X_i), Y_i) = 1, \text{if} \quad \hat{C}_{\lambda}^{\text{low}}(X_i) > 95 \quad \text{and} \quad Y_i< 95, \notag\\
&\ell (\hat{C}_{\lambda}(X_i), Y_i) = 0, \quad \text{otherwise}. \label{eq:loss}
\end{align}

 Similar to previous conformal methods, after training a base model using $X_{train}$ and $y_{train}$, we can use it to make predictions $\hat{y}_{val}$ from $X_{val}$. We calculate nonconformity scores defined as $|\hat{y}_{val} - y_{val}|$. To create our confidence intervals, we choose an $\alpha$ level that represents the desired risk level. Afterwards, we search among a series of $\lambda$ values and utilize $\lambda$ to generate predictive intervals in the validation data as $ \hat{C}_{\lambda}(X_i) =  [\hat{y}_{val} - \lambda* err, \hat{y}_{val} + \lambda* err]$, where 
 \begin{align}
 err = \max |\hat{y}_{val}  - y_{val}|, \notag
 \end{align}
 which is the largest nonconformity score in the validation data.
  For each of $\lambda$, applying the risk function to each data point, we calculate the average risk score on the validation data:
 \begin{align}
 \hat{r}(\lambda) = \frac{1}{n} \sum_i \ell (\hat{C}_{\lambda}(X_i), Y_i). 
 \end{align}

We then choose $\lambda$ when the following is reached:  

\begin{equation}
   \lambda = \inf \{ \lambda: \frac{n}{n+1} \hat{r}(\lambda) + \frac{1}{n+1}\leq \alpha \}, \label{eq:risk_control}
\end{equation}

where $n$ is the number of samples in the calibration data. For more derivation and justification of this formula, we refer the reader to our Appendix \ref{app:derivation}. 
% We set the last $\lambda$ value that does not satisfy this equation as $I$. 
Finally, we use our model to generate predictions for $\hat{y}_{test}$ based upon $X_{test}$ and set our confidence interval $\hat{C}$ as 
\begin{equation}
    \hat{C}_{\lambda} = [\hat{y}_{test} - \lambda*err, \hat{y}_{test} + \lambda*err].
\end{equation}

\subsection{Conformal Training}
In conformal training, the goal is to achieve sharper confidence intervals by utilizing penalty terms associated with the confidence interval during the model training \citep{Stutz}. In normal conformal prediction, the model is trained on the training dataset and then conformal prediction is applied on the validation dataset in a post hoc manner. As a result, the confidence interval is significantly influenced by the trained model but there is no mechanism within the conformal framework to further improve it. In conformal training, the main idea is to incorporate a penalty term associated with the generated confidence interval on the training data in the training process. So after splitting the data, we perform training using each mini-batch of training data with an additional conformal prediction step utilizing validation data and penalize the distance between the lower bound of the confidence interval and the true labels, essentially the size of the interval. Therefore, the model trained in this way will be specifically optimized to generate smaller confidence intervals.
% In particular, the penalty term is:
% \begin{align}
    
% \end{align}
By doing so, the model training process is influenced and tailored by the resulting confidence interval, in addition to the original learning loss. 
After training, in our application, we adapt the original conformal training method, as we take the average of the one-side confidence interval during training as $I$. Then, we use this model to make predictions of $\hat{y}_{test}$ based on $X_{test}$ and set our confidence interval to be  $\hat{C} = [\hat{y}_{test} - I, \hat{y}_{test} + I]$. Note that it is a heuristic to average the interval during training. But given validation data is involved in the generation of the confidence interval in each step, it is similar to taking an average of the calibrated interval of an ensembled model. However, this process does not take any risk functions into account. So we cannot directly control the sensitivity, which is the most important metric for the safety of the IMRT QA. Also, conformal training usually loses the coverage guarantee, if there is no additional calibration step. 

\subsection{Training-Aware Conformal Risk Control for IMRT QA}
In our proposed method of training-aware conformal risk control method, we propose to utilize the risk functions actually used in clinical practices in risk control and also incorporate it as the penalty term in conformal training. Specifically,
different from the original conformal training \citep{Stutz},
which set the objective function as MSE + max(0, (lower conformal bound - actual)) + max(0, (actual - upper conformal bound)) with the conformal bounds generated by conformal prediction, we use only the lower bound penalty, which is MSE + max(0, (lower conformal bound - actual value)), and use conformal risk control prediction to generate tighter bands that have a guarantee in risk control for each iteration of conformal training. 
we set the penalty term like how we specify the risk function in the conformal risk control Eq. \eqref{eq:loss}. As for the training process, we operate under a similar protocol as in conformal training in the previous section, but create our confidence intervals for each minibatch via conformal risk control, namely selecting a different $\lambda$ in each minibatch. We take the average of the one-side of confidence interval during training to be $I$. Then, we use this model to make predictions of $\hat{y}_{test}$ based on $X_{test}$ and set our confidence interval to be $\hat{C} = [\hat{y}_{test} - I, \hat{y}_{test} + I]$. Even though we may lose the coverage guarantee in theory, we empirically maintain the coverage, as shown in our experiments. Moreover, as we will discuss later in our paper, over-prediction in our application, especially when the true GPR is low, is much more harmful than under-prediction. Hence coverage may not be the most important metric.

% \subsection{ML-assisted IMRT QA process}
After we obtain the confidence interval, we take a conservative approach to make decisions. We will only classify a treatment plan as safe when the lower bound of the prediction interval is safe. Traditionally, this ``safe'' vs ``unsafe'' decision is usually based on a threshold on a point estimator. Here, to be safe, we utilize the lower bound of the conformal prediction interval. This conservative approach is built into our conformal risk control and training-aware conformal risk control framework by leveraging a risk function like Eq. \eqref{eq:loss}. 

Our objective is to reduce the number of ``safe'' plans that are measured without letting the ``unsafe'' plans get delivered without measurement. So we want a high sensitivity with a high specificity, as well as a good reduction of measurement.
Here the calculation of the reduction of measurement is
\begin{align}
 \frac{\text{Number\ of\ Points\ Predicted\ as\ Safe}}{\text{Total\ Number\ of\ Test\ Points}}.
\end{align}
This is under the assumption that in the ideal case the model is trustworthy and physicists fully trust the model in an ML-assisted human decision-making scenario. In practice, there may exist potential human-AI trust issues, which we will refer to future work. 

% 1. mention how the reduction of measurement is calculated\\
% 2. mention how we are using the very conservative (lower bound) in decision making --> how do you get safe unsafe from prediction intervals\\
% 3. mention this should not be a fully automatic process, but a ML-assisted decision making, assuming human trust the model, we will study the human AI trust problem in future work

\section{Experiments}
\subsection{Datasets}
In total, we use 2 IMRT plan datasets derived from across the Johns Hopkins hospital network in 2023. \textit{Dataset 1} was derived from cases from dates between 6/23 and 8/23. The dataset contains 394 patient plans delivered with the machine Elekta VersaHD and 4 plans are below 95\% GPR on a 3\%/3mm criteria. \textit{Dataset 2} was derived from cases from dates between 9/23 and 12/23. The dataset has 594 patient plans delivered with the machine Elekta VersaHD out of which 15 plans had values below 95\% GPR on a 3\%/3mm criteria. So all plans in \textit{Dataset 1} and \textit{Dataset 2} were performed at the same site on the same machine with different timeframes. We notice that the ratio of the plans below 95\% GPR (label distribution) is very small and also different between two datasets. The patient and plan characteristics, which are the covariates, can also be different. However, the data generating mechanism is totally the same. Specifically, we highlight that the dataset comes from  dosimetrically matched machines, i.e. they are clinically interchangeable. We match the machines based on rigorous physics criteria and believe this facilitates a strengthening of the dataset. So we conduct two sets of experiments: 1) pooling data from \textit{Dataset 1} and \textit{Dataset 2} and random split them into train, calibration, and test, to test the methods in homogeneous data distributions; 2) training and calibrating on one of them and test on the other, to test the methods under data distribution shift. 
% \textit{Dataset 3} was derived from Site 2. The dataset had 121 patient plans delivered with machine Elekta Infinity and 2 plans had values below 95\% gamma passing rate on a 3\%/3mm criteria. \textit{Dataset 4} was derived from Site 3. The dataset had 173 patient plans delivered with Elekta Axesse and 2 plans had values below 95\% gamma passing rate on a 3\%/3mm criteria. All datasets were standardized to have 32 features in total. The gamma passing rates were calculated with RayStation.

\subsection{Experimental Setup}
% We performed two sets of experiments. For each set of experiments, 
In our experiment, we train a baseline model before we apply conformal prediction \citep{Angelopoulos-1}, conformal quantile regression \citep{Romano}, conformal risk control \citep{Angelopoulos-2}, conformal training \citep{Stutz}, and our training-aware conformal risk control method. Our base model is an ensemble model from 5 runs of training, each with different model initializations. We also run all these conformal methods 3 times across different train-validation-test splits to generate means and error bars of the result. We take a conservative approach to generate the upper bound of the confidence interval as the largest value of the upper bound estimates made by each model in the ensemble. Similarly, the lower bound of the confidence interval would be the smallest value of the upper bound estimates made by each model in the ensemble. 

In our first set of experiments, we study how different conformal prediction methods performed on a dataset created by pooling \textit{Dataset 1} and \textit{Dataset 2}. We were able to pool the datasets since the plans in both datasets were gathered from the same machine with the same data gathering methodology. We split the pooled dataset into train, validation, and test datasets and then compare the results of conformal prediction, conformal quantile regression, conformal risk control, conformal training, and our training-aware conformal risk control method.

In the second set of experiments, we study how well different conformal methods can deal with potential distribution shifts. In the first part of these experiments, we use \textit{Dataset 1} for our training and validation dataset and \textit{Dataset 2} for testing. 
% We assess how well conformal prediction, conformal quantile regression, conformal risk control, conformal training, and our training-aware conformal risk control method perform. 
In the second part of this experiment, we use \textit{Dataset 2} for our training and validation datasets and \textit{Dataset 1} for our testing dataset. In both cases, we perform the exact same conformal prediction methods as in the pooled data. Note in this case, the calibration data is not exchangeable with test anymore.

% In our experiments, we examine whether conformal prediction methods can successfully predict confidence intervals that can better help triage plans when the training, validation and testing datasets come from the same machine. In this scenario, we used \textit{Dataset 1} for training and validation and we tested on \textit{Dataset 2}. 

% In the second set of experiments, we examine whether conformal prediction methods can be generalized successfully to predict confidence intervals that can better help triage plans when the testing dataset comes from different machines. In this scenario, we used \textit{Dataset 1} for training, \textit{Dataset 2} for validation and we tested on \textit{Dataset 3} or \textit{Dataset 4}.

\begin{table*}[t]
\centering
\small
    \caption{\textit{Pooled Dataset 1 and Dataset 2} Result with a Prospective 95\% Threshold}
    % \begin{tabular}{|>{\centering\arraybackslash}m{2cm}|>{\centering\arraybackslash}m{2cm}|>{\centering\arraybackslash}m{2cm}|>{\centering\arraybackslash}m{3cm}|>
    % {\centering\arraybackslash}m{2cm}|>
    % {\centering\arraybackslash}m{2cm}|}
    \begin{tabular}{cccccc}
        \hline
        \textbf{Method} & \textbf{Sensitivity} & \textbf{Specificity} & \textbf{Reduction in Measurement} & \textbf{Coverage} & \textbf{ Interval Width} \\ \hline
        Base model & 0 $\pm$ 0 & 0.97 $\pm 0.01$ & 0.95 $\pm$ 0.01 & NA & NA\\ \hline
        CP & 1 $\pm$ 0 & 0 $\pm$ 0.01 & 0 $\pm$ 0.01 & 0.99 $\pm$ 0  & 12.29 $\pm$ 0.18 \\ \hline
        CQR & 1 $\pm$ 0  & 0 $\pm$ 0  & 0 $\pm$ 0  & 0.99 $\pm$ 0.00  & 6.88 $\pm$ 0.23 \\ \hline
        CRC & 1 $\pm$ 0 & 0.28 $\pm$ 0.09  & 0.27 $\pm$ 0.09 & 0.98 $\pm$ 0 & 10.46 $\pm$ 0.14 \\ \hline
        CT &  1 $\pm$ 0  & 0.1 $\pm$ 0.08 & 0.1 $\pm$ 0.08 & 0.98 $\pm$ 0.01  & 12.32 $\pm$ 0.73 \\ \hline
        Ours & 1 $\pm$ 0 & 0.76 $\pm$ 0.04 & 0.75 $\pm$ 0.04 & 0.94 $\pm$ 0.02 & 5.81 $\pm$ 0.22\\ \hline
        \hline
    \end{tabular}
    \label{tab:pool_1}
\end{table*}

\begin{table*}[htbp]
\centering
\small
    \caption{\textit{Pooled Dataset 1 and Dataset 2} Result with a Retrospective Threshold}
    % \begin{tabular}{|>{\centering\arraybackslash}m{2cm}|>{\centering\arraybackslash}m{2cm}|>{\centering\arraybackslash}m{2cm}|>{\centering\arraybackslash}m{3cm}|>
    % {\centering\arraybackslash}m{2cm}|>
    % {\centering\arraybackslash}m{2cm}|}
    \begin{tabular}{cccccc}
        \hline
        \textbf{Method} & \textbf{Sensitivity} & \textbf{Specificity} & \textbf{Reduction in Measurement} & \textbf{Coverage} & \textbf{ Interval Width} \\ \hline
        Base model & 1 $\pm$ 0 & 0.59 $\pm 0.13$ & 0.58 $\pm$ 0.13 & NA & NA\\ \hline
        CP & 1 $\pm$ 0 & 0.71 $\pm$ 0.07 & 0.69 $\pm$ 0.07 & 0.99 $\pm$ 0  & 12.29 $\pm$ 0.18 \\ \hline
        CQR & 1 $\pm$ 0  & 0.8 $\pm$ 0  & 0.78 $\pm$ 0  & 0.99 $\pm$ 0.00  & 6.88 $\pm$ 0.23 \\ \hline
        CRC & 1 $\pm$ 0  & 0.54 $\pm$ 0.09  & 0.53 $\pm$ 0.09 & 0.98 $\pm$ 0 & 10.46 $\pm$ 0.14 \\ \hline
        CT &  1 $\pm$ 0  & 0.6 $\pm$ 0.28 & 0.58 $\pm$ 0.27 & 0.98 $\pm$ 0.01  & 12.32 $\pm$ 0.73 \\ \hline
        Ours & 1 $\pm$ 0 & 0.82 $\pm$ 0.05 & 0.83 $\pm$ 0.06 & 0.94 $\pm$ 0.02 & 5.81 $\pm$ 0.22\\ \hline
        \hline
    \end{tabular}
    \label{tab:pool_2}
\end{table*}

\subsection{Evaluation Metrics}
To evaluate every conformal prediction method, we analyze the sensitivity, specificity, and reduction in measurement based upon a 95\% GPR threshold as well as retrospectively for the threshold that results in the highest specificity while maintaining the highest sensitivity. The 95\% GPR under 3\%/3mm criteria is selected because 95\% correspond with the universal tolerance and universal action limits and 3\%/3mm criteria is widely accepted. Additionally, we measure the coverage of each method and the average interval width. With high sensitivity, specificity, and reduction in measurement, we would expect the method to achieve high coverage and small interval width.

\subsection{Data Preprocessing}

% The complexity features of the plan were calculated using Python 3.9 in conjunction with the Pydicom package and extracted as inputs for the AI-assisted IMRT QA models.
We use the complexity features of the plan as the input to our models. Table \ref{table:abbreviations} in appendix lists the complexity metrics calculated from each treatment plan with their definitions. Specifically, for the definitions and calculation of the features, we follow previous work \citep{Lam}.
For different analysis criteria, the GPR was recorded for each individual treatment plans and recorded in a corresponding entry in the dataset utilized for the machine learning model input. Note that there are cases when an individual patient has multiple treatment plans. But given each plan is independent and labeled separately, there would not be information leakage even if the same patient's plan appears in training and testing.

\textbf{Feature Selection} In our data exploration, we realize further feature selection can be helpful to further improve the model performance (see Appendix \ref{app:selection} for more details). To select the variables used to train our model, we operate on the assumption that there would be differences in variable distributions between plans that passed the 95\% gamma threshold and those that did not. Our hypothesis is, if a feature $v$ is differentiating between the two classes, ``safe'' ($y=0$) v.s. ``unsafe'' ($y=1$), the distribution of $P(v|y=0)$ should be significantly different from $P(v|y=1)$.  We then perform a 2-sample t-test to look for variables that are statistically different (p
0.05). We find that 12 features have a statistical difference between their distributions between ``safe'' and ``unsafe''. They are PAAJA, PEM, Pgantryvel, PI, PmaxAP\_v, PMAXJ, PmaxnRegs, PMCS, PminAP\_va, PMSAS2, PMUCP, and PuniaccMLC. Detailed description of the features can be found in Appendix \ref{app:features}.  

% \subsubsection{Training Dataset Class Balancing}

\subsection{Base Model Training and Hyperparameter Tuning}
The base model is a two-layer MLP for regression in all of our methods. The loss function for training this base model is the mean squared error. All the conformal methods share the same base model that is tuned by hyperparameter tuning. 
During hyper-parameter tuning, we vary the number of nodes in our hidden input layer (50, 100, 200), the activation function (ReLU or Sigmoid), the number of epochs (500, 1000, 1500), and the learning rate (0.1, 0.01, 0.001). We perform a grid search to find hyperparameters that minimize the mean squared error in the validation data. 
All developed models are then trained using hyper-parameters from the best baseline model. More details about the model can be found in the Appendix \ref{apd:first}.

Since there is a huge class imbalance if we consider a passing rate greater than 95\% as ``safe'' and a passing rate below 95\%  as ``unsafe'', we balance the number of plans that were below 95\% and above 95 \% in our training dataset to achieve balanced class weighting in the training. Note that, to avoid data leakage in calibration and evaluation, we do not balance calibration data and test data, hence making sure they satisfy the exchangability assumption in conformal prediction (in the pooled data cases).

\begin{table*}[htbp]
\centering
\small
    \caption{\textit{Dataset 1} Result with a Prospective 95\% Threshold}
    % \begin{tabular}{|>{\centering\arraybackslash}m{2cm}|>{\centering\arraybackslash}m{2cm}|>{\centering\arraybackslash}m{2cm}|>{\centering\arraybackslash}m{3cm}|>
    % {\centering\arraybackslash}m{2cm}|>
    % {\centering\arraybackslash}m{2cm}|}
    \begin{tabular}{cccccc}
        \hline
        \textbf{Method} & \textbf{Sensitivity} & \textbf{Specificity} & \textbf{Reduction in Measurement} & \textbf{Coverage} & \textbf{ Interval Width} \\ \hline
        Base model & 0.31 $\pm$ 0.10 & 0.97 $\pm 0.01$ & 0.94 $\pm$ 0.01 & NA & NA\\ \hline
        CP & 1 $\pm$ 0 & 0 $\pm$ 0 & 0 $\pm$ 0 & 0.99 $\pm$ 0.01  & 14.56 $\pm$ 0.54 \\ \hline
        CQR & 1 $\pm$ 0  & 0 $\pm$ 0  & 0 $\pm$ 0  & 0.99 $\pm$ 0  & 6.86 $\pm$ 0.31 \\ \hline
        CRC & 0.91 $\pm$ 0.01 & 0.27 $\pm$ 0.17  & 0.26 $\pm$ 0.16 & 0.98 $\pm$ 0 & 9.87 $\pm$ 0.49 \\ \hline
        CT &  0.86 $\pm$ 0.05  & 0.45 $\pm$ 0.06 & 0.44 $\pm$ 0.05 & 0.92 $\pm$ 0.08  & 5.56 $\pm$ 4.40 \\ \hline
        Ours & 0.86 $\pm$ 0.05 & 0.68 $\pm$ 0.04 & 0.66 $\pm$ 0.04 & 0.94 $\pm$ 0.02 & 5.81 $\pm$ 0.22\\ \hline
        \hline
    \end{tabular}
    \label{tab:data1_1}
\end{table*}

\begin{table*}[htbp]
\centering
\small
    \caption{\textit{Dataset 1} Result with a Retrospective Threshold}
    % \begin{tabular}{|>{\centering\arraybackslash}m{2cm}|>{\centering\arraybackslash}m{2cm}|>{\centering\arraybackslash}m{2cm}|>{\centering\arraybackslash}m{3cm}|>
    % {\centering\arraybackslash}m{2cm}|>
    % {\centering\arraybackslash}m{2cm}|}
    \begin{tabular}{cccccc}
        \hline
        \textbf{Method} & \textbf{Sensitivity} & \textbf{Specificity} & \textbf{Reduction in Measurement} & \textbf{Coverage} & \textbf{ Interval Width} \\ \hline
        Base model & 0.92 $\pm$ 0 & 0.25 $\pm 0.20$ & 0.24 $\pm$ 0.19 & NA & NA\\ \hline
        CP & 1 $\pm$ 0 & 0.03 $\pm$ 0.01 & 0.03 $\pm$ 0.01 & 0.99 $\pm$ 0.01  & 14.56 $\pm$ 0.54 \\ \hline
        CQR & 0.92 $\pm$ 0  & 0.66 $\pm$ 0.03  & 0.64 $\pm$ 0.03  & 0.99 $\pm$ 0  & 6.86 $\pm$ 0.31 \\ \hline
        CRC & 0.97 $\pm$ 0.05  & 0.18 $\pm$ 0.24  & 0.17 $\pm$ 0.23 & 0.98 $\pm$ 0 & 9.87 $\pm$ 0.49 \\ \hline
        CT &  1 $\pm$ 0  & 0.18 $\pm$ 0.15 & 0.18 $\pm$ 0.15 & 0.92 $\pm$ 0.08  & 5.56 $\pm$ 4.40 \\ \hline
        Ours & 1 $\pm$ 0 & 0.33 $\pm$ 0.14 & 0.32 $\pm$ 0.14 & 0.94 $\pm$ 0.02 & 5.81 $\pm$ 0.22\\ \hline
        \hline
    \end{tabular}
    \label{tab:data1_2}
    \vspace{-10pt}
\end{table*}

\section{Results}
% \subsection{Evaluation Methodology}
 % For each conformal prediction methodology, we recorded the empirical coverage. We also calculated each conformal prediction methodology's performance for sensitivity, specificity, and reduction in measurement using a retrospective threshold to show optimal performance in optimizing sensitivity. Calculations for these metrics are defined in Appendix B.  

% \subsubsection{Validation Experiment}

We compare our proposed method with baselines, including base model (regression only, no conformal prediction), standard split conformal (CP), conformal quantile regression (CQR), conformal risk control (CRC), and conformal training methods (CT).

Table \ref{tab:pool_1} and Table \ref{tab:pool_2} show the results of using a pooled data and randomly split them into train, validation, and test.  Table \ref{tab:data1_1} and Table \ref{tab:data1_2} show results of
using \textit{Dataset 2} to train and validate, and \textit{Dataset 1} to test. The results of using \textit{Dataset 1} to train and validate, and \textit{Dataset 2} to test can be found in Appendix \ref{app:results}. 
% The predictions were evaluated for sensitivity, specificity, reduction of measurement against a 95\% gamma passing rate threshold, the universal tolerance limit as defined by AAPM. 
From all the tables, we can easily see the drawbacks of base models only, namely not using conformal prediction methods. When using a prospective 95\% GPR, the sensitivity is very low for the base model. When using a retrospective threshold, the specificity of base model is much lower than other methods. With conformal prediction methodologies, we are able to guarantee 100\% sensitivity in the pooled data, but may not maintain it under data distribution shift. However, all methods manage to achieve coverage above 0.9, even under distribution shifts. Overall, the performance under distribution shift is worse than the pooled data, highlighting the importance of exchangability assumption in the conformal prediction.

CP has the widest band among all methods. While it managed to achieve a high sensitivity, it has a lower specificity and reduction of measurement than the proposed method.  
We can see that there is an improvement in CQR's performance over CP in specificity and reduction in measurement under retrospective thresholds in both pooled data and shifted data. But it still underperforms the proposed method. 
CRC manages to increase the specificity and reduction in measurement considerably over CQR in the prospective threshold case but is worse in the retrospective threshold case. CT, being designed to learn to make better predictions by taking into account the width of conformal predictions,  surprisingly, does not shrink the interval width much. Its specificity and measurement reduction is also underperforming the proposed method. Finally, with our new method of training-aware conformal risk control, we manage to achieve the highest specificity and reduction in measurement, while maintaining 100\% sensitivity and a relatively small interval width in the pooled data.  Our specificity and measurement reduction is still competitive under distribution shift in Table \ref{tab:data1_1} and Table \ref{tab:data1_2} but the sensitivity is degraded in the prospective threshold case.

\section{Conclusion and Discussion}
In this work, we propose to leverage methods in the conformal prediction framework to develop a machine-learning-assisted IMRT QA and treatment plan triage solution. Conformal prediction methods are particularly suitable for this problem because IMRT QA is a safety-critical task whose risk needs to be carefully controlled. We propose a training-aware conformal risk control method by incorporating the conformal risk control into the training process. We also incorporate the actual decision-making threshold on
the GPR and the risk functions
used in clinical evaluation into the design of the
risk control framework. We compare with various conformal methods in baselines and our method achieves
high sensitivity and specificity, significantly
reducing the measurement without generating a
huge confidence interval. Our results demonstrate the validity and applicability of conformal prediction methods for improving efficiency
and reducing the workload of the IMRT QA process.

% \section{Discussion}

\textbf{Limitation of the small data and distribution shift}
One limitation of our work is our data sample size is relatively small. 
We collected data in a relatively short timeframe, due to clinical reasons. The clinical software from which our treatment plans are created was recently commissioned near March 2023 and thus our dataset was limited in this regard. We aim to collect data in a longer time frame in the future to further validate our proposed methods. In the future, we also aim to incorporate conformal prediction methods developed specifically under data distribution shift \citep{tibshirani2019conformal,prinster2023jaws,gibbs2021adaptive,prinster2022jaws,podkopaev2021distribution} to further improve our model.

\textbf{Clinical Deployment and Future Improvements}
In practice, medical physicists may have their own criteria for determining the safety beyond using a GPR. For example, they may want to utilize visualizations of different treatment regions and disease sites, because the same GPR may mean very different quality when it comes to sensitive regions and insensitive regions of the patient body. Therefore, we aim to develop more interpretable prediction models with visualizations to improve the ML-assisted decision-making process in future work. 

\section*{Acknowledgments}
We would like to thank the anonymous reviewers for their helpful comments. AL was partially supported by the Amazon Research Award, the Discovery Award of the Johns Hopkins University, a seed grant from the JHU Institute of Assured Autonomy (IAA), and a seed grant from the JHU Center for Digital Health and Artificial Intelligence (CDHAI).

% will never be able to reach the optimal threshold. Medical physicists will have to make a determination of the threshold for an acceptable gamma passing rate for the lower bound and what percentile of the calibration dataset they want to base the error upon. These determinations can shape the width of error bands and the assessment of accuracy. Generally, it would be better to set a more conservative lower band to ensure patient safety.

% Future work needs to focus upon improving the baseline model to prevent missing cases. Generally, the cases that tend to be missed in the current setting are stereo cases which are higher dosage plans. An ensemble-based model that differentiates between stereo and non stereo cases will probably resolve the issue and allow the conformal prediction to perform better with tighter intervals. 
% Further, conformal training shows quite a bit of promise on the Sibley dataset. Given its promise, one way in which we minimize inefficiency but also improve performance might be to apply conformal risk control to it as well.

\bibliography{references}

\appendix
\newpage
\onecolumn

\section{More Illustration of the IMRT Process}
To further illustrate the IMRT process, we demonstrate the following figures. Figure depicts instance of an aperture shape and ring depicts the normalized output as a function of rotational gantry position. Our complexity features describe the treatment plan. More details about the features can be found in Appendix \ref{app:features}. 

\begin{figure}[h]
\begin{tabular}{cc}
\centering
\includegraphics[width=0.48\textwidth]{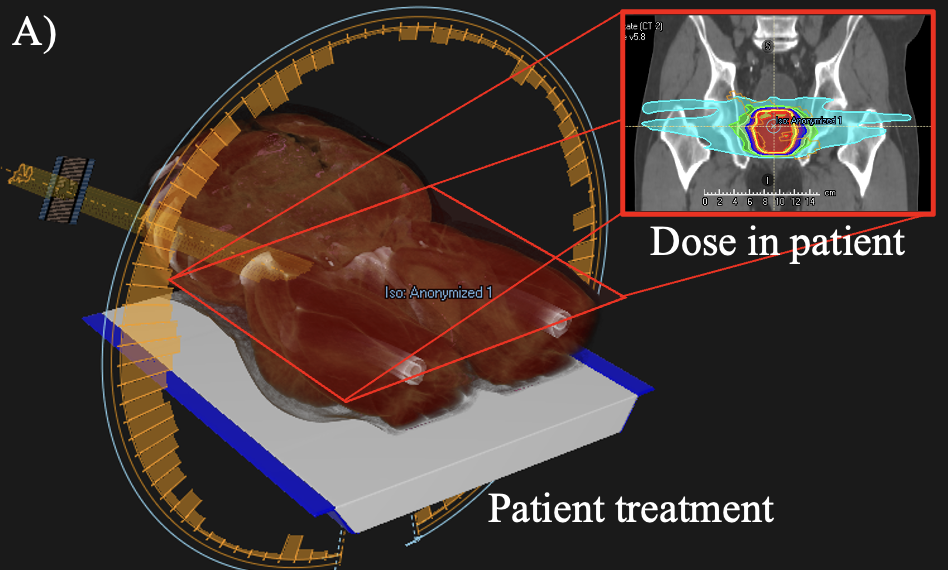}&
\includegraphics[width=0.48\textwidth]{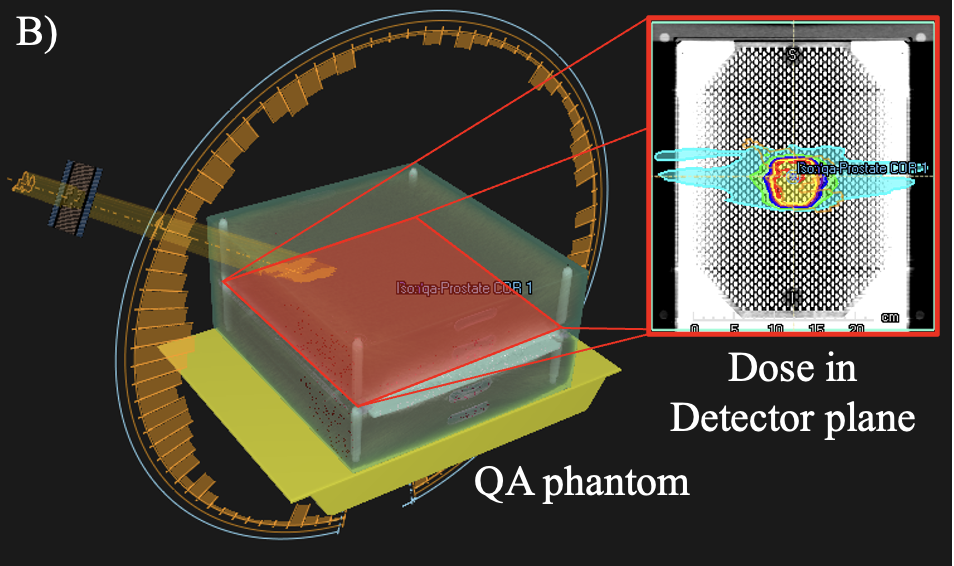}
\end{tabular}
\caption{Depiction of coronal absorbed dose distribution in the A) patient and correspondingly in the B) quality assurance phantom detector plane. }
\end{figure}

\section{Feature Descriptions and More Details on Feature Extraction}
\label{app:features}
As mentioned, we follow the feature definition and calculation in the previous work \citep{Lam}. We demonstrate their feature names and descrptions in Table \ref{table:abbreviations}.
\begin{table*}[h]
\centering
% \scriptsize
\begin{tabular}{|p{3cm}|p{12cm}|}
\hline
\textbf{Abbreviation} & \textbf{Description} \\
\hline
BA & Beam aperture area weighted by MU \\
BI & Beam irregularity \\
BM & Fraction of BA normalized by UAA \\
UAA & Union area of aperture (UAA) \\
MFAS2,5,10,20 & Mean of fraction of aperture smaller (MFAS) than 2, 5, 10, 20 mm \\
MaxFAS2,5,10,20 & Max of fraction of aperture smaller (MaxFAS) than 2, 5, 10, 20 mm \\
MAA & Mean aperture area \\
MAD & Maximum distance of the mid-point between any open leaf-pair in a beam \\
MUCP & Mean of MUs per control point in a beam \\
MLO1,2,3,4,5 & Moment order of 1, 2, 3, 4, 5 of leaf openings \\
minAP\_h & Minimum aperture perimeter in horizontal direction \\
maxAP\_h & Maximum aperture perimeter in horizontal direction \\
minAP\_v & Minimum aperture perimeter in vertical direction \\
maxAP\_v & Maximum aperture perimeter in vertical direction \\
maxRegs & Maximum number of regions in the beam \\
AAJA & Ratio of the average area of an aperture over the area defined by jaws \\
MAXJ & Maximum of x-y jaw positions \\
MCS & Modulation complexity score \\
EM & Edge metric: ratio of MLC side-length to aperture area \\
\hline
\end{tabular}
\caption{Description of Feature Variables}
\label{table:abbreviations}
\end{table*}

\section{Feature Selection Process}\label{app:selection}
In our data exploration and feature selection process, we experiment with different methods. Table \ref{tab:feature} shows our results when we use different models on \textit{Dataset 2}. We can see that feature selection utilizing our method is outperforming random forests and Lasso regression using the full features.

\begin{table*}[h]
\centering
\small
\caption{\textit{Results comparison for different feature selection methods using a 99.99 threshold}}
\label{tab:feature}
\begin{tabular}{|c|c|c|c|}
\hline
Method & Sensitivity & Specificity & Reduction in Measurement \\ 
\hline
Full Variable Random Forest & 1 & 0.03 & 0.03 \\
\hline
Full Variable ElasticNet & 1 & 0.15 & 0.14 \\
\hline
Full Variable MLP & 0.73 & 0.33 & 0.32 \\
\hline
Full Variable LassoRegressor & 1 & 0 & 0 \\
\hline 
Feature Selected MLP & 1 & 0.33 & 0.32\\
\hline
\end{tabular}
\end{table*}

\section{Model Hyperparameters}\label{apd:first}
In our experiments, we conduct grid search to find the best parameters for training the base model. The resulting hyperparameters we use are as follows:
% \begin{enumerate}
%     \item
Hidden Nodes: 100, Activation Function: sigmoid, Epochs: 1500, Learning Rate: 0.01
% \end{enumerate}

In conformal methods, we use the following miscoverage level or risk control level to generate the confidence intervals:

\begin{enumerate}
    \item Conformal Prediction:  Miscoverage level $\alpha=0.1$ 
    \item Conformal Quantile Regression: Percentile: [5,95], Miscoverage level $\alpha=0.1$ 
    \item Conformal Risk Control:  Risk level $\alpha=0.1$ 
    \item Conformal Training: Miscoverage level $\alpha=0.1$ 
    \item The proposed method: Risk level $\alpha=0.1$ 
\end{enumerate}

\section{Evaluation Metrics}\label{apd:second}

We define sensitivity, specificity, reduction in measurement, and coverage as follows:

Sensitivity: 
\begin{equation}
    \frac{\text{Number\ of\ points\ with\ Predicted\ Gamma\ Passing\ Rate\ below\ 95\ and\ Actual\ Gamma\ Passing\ Rate\ below\ 95}}{\text{Total\ Number\ of\ Actual\ Gamma\ Passing\ Rate\ below\ 95}}
\end{equation}
Specificity:
\begin{equation}
    \frac{\text{Number\ of\ points\ with\ Predicted\ Gamma\ Passing\ Rate\ above\ 95\ and\ Actual\ Gamma\ Passing\ Rate\ above\ 95}}{\text{Total\ Number\ of\ Actual\ Gamma\ Passing\ Rate\ above\ 95}}
\end{equation}

Reduction in Measurement: 

\begin{equation}
    \frac{\text{Number\ of\ Points\ Predicted\ as\ Safe}}{\text{Total\ Number\ of\ Test\ Points}}
\end{equation}

Coverage:
\begin{equation}
\frac{\text{Number\ of\ Actual\ Points\ within\ Upper\ and\ Lower\ Prediction\ Band}}{\text{Total\ Number\ of\ Test\ Points}}    
\end{equation}

\section{Additional Experimental Results}
\label{app:results}
We have also conducted experiment on training with \textit{Dataset 1} and testing with \textit{Dataset 2}. Table \ref{tab:data2_1} and \ref{tab:data2_2} demonstrate the results. It further shows how methods can be influenced by data distribution shift. For most methods, their specificity and reduction in measurement are much lower than the pooled data cases. 
\begin{table*}[h]
\centering
\small
    \caption{\textit{Dataset 2} Results with a Prospective 95\% Threshold}
    % \begin{tabular}{|>{\centering\arraybackslash}m{2cm}|>{\centering\arraybackslash}m{2cm}|>{\centering\arraybackslash}m{2cm}|>{\centering\arraybackslash}m{3cm}|>
    % {\centering\arraybackslash}m{2cm}|>
    % {\centering\arraybackslash}m{2cm}|}
    \begin{tabular}{cccccc}
        \hline
        \textbf{Method} & \textbf{Sensitivity} & \textbf{Specificity} & \textbf{Reduction in Measurement} & \textbf{Coverage} & \textbf{Interval Width} \\ \hline
        Base model & 0.18 $\pm$ 0.04 & 0.96 $\pm$ 0.01 & 0.94 $\pm$ 0.01 & NA & NA\\ \hline
        CP & 1 $\pm$ 0 & 0 $\pm$ 0.01 & 0 $\pm$ 0.01 & 1 $\pm$ 0 & 12.96 $\pm$ 0.24\\ \hline
        CQR & 1 $\pm$ 0 & 0.03 $\pm$ 0 & 0.03 $\pm$ 0 & 0.99 $\pm$ 0 & 10.51 $\pm$ 0.05\\ \hline
        CRC & 1 $\pm$ 0 & 0.37 $\pm$ 0.12 & 0.36 $\pm$ 0.11 & 0.98 $\pm$ 0.01 & 9.71 $\pm$ 0.34 \\ \hline
        CT & 1 $\pm$ 0 & 0 $\pm$ 0 & 0 $\pm$ 0 & 0.98 $\pm$ 0 & 15.56 $\pm$ 0.55 \\ \hline
        Ours & 1 $\pm$ 0 & 0.33 $\pm$ 0.12 & 0.32 $\pm$ 0.11 & 0.90 $\pm$ 0.02 & 8.34 $\pm$ 0.60 \\ \hline
    \end{tabular}
    \label{tab:data2_1}
\end{table*}    

\begin{table*}[htbp]
\centering
\small
    \caption{\textit{Dataset 2} Result with a Retrospective Threshold}
    % \begin{tabular}{|>{\centering\arraybackslash}m{2cm}|>{\centering\arraybackslash}m{2cm}|>{\centering\arraybackslash}m{2cm}|>{\centering\arraybackslash}m{3cm}|>
    % {\centering\arraybackslash}m{2cm}|>
    % {\centering\arraybackslash}m{2cm}|}
    \begin{tabular}{cccccc}
        \hline
        \textbf{Method} & \textbf{Sensitivity} & \textbf{Specificity} & \textbf{Reduction in Measurement} & \textbf{Coverage} & \textbf{ Interval Width} \\ \hline
        Base model & 0.98 $\pm$ 0.04 & 0.57 $\pm 0.05$ & 0.58 $\pm$ 0.08 & NA & NA\\ \hline
        CP & 1 $\pm$ 0 & 0.35 $\pm$ 0.12 & 0.34 $\pm$ 0.12 & 1 $\pm$ 0 & 12.96 $\pm$ 0.24\\ \hline
        CQR & 1 $\pm$ 0  & 0.22 $\pm$ 0.03  & 0.21 $\pm$ 0.02  & 0.99 $\pm$ 0  & 10.51 $\pm$ 0.05 \\ \hline
        CRC & 1 $\pm$ 0 & 0.43 $\pm$ 0.09 & 0.42 $\pm$ 0.09 & 0.98 $\pm$ 0.01 & 9.71 $\pm$ 0.34 \\ \hline
        CT & 1 $\pm$ 0 & 0.36 $\pm$ 0.02 & 0.35 $\pm$ 0.02 & 0.98 $\pm$ 0 & 15.56 $\pm$ 0.55 \\ \hline
        Ours & 1 $\pm$ 0 & 0.47 $\pm$ 0.06 & 0.45 $\pm$ 0.07 & 0.90 $\pm$ 0.02 & 8.34 $\pm$ 0.60 \\ \hline
        \hline
    \end{tabular}
    \label{tab:data2_2}
\end{table*}
% For our baseline models predicting

\section{Conformal Risk Control Derivation} \label{app:derivation}

We include the derivation of the formula in Conformal Risk Control Eq. (\ref{eq:risk_control}) for reader's reference. Much of the content is directly from \citep{Angelopoulos-2}.

\begin{theorem}
Assume that the loss function is defined as:
\[
\ell(C_\lambda(X_i), Y_i) = 
\begin{cases} 
1, & \text{if } C_{\text{low},\lambda}(X_i) > 95 \text{ and } Y_i < 95, \\
0, & \text{otherwise},
\end{cases}
\]
and that $\ell(C_\lambda(X_i), Y_i)$ is non-increasing in $\lambda$, right-continuous, and
\[
\ell(C_\lambda(X_i), Y_i) \leq \alpha, \quad \sup_\lambda \ell(C_\lambda(X_i), Y_i) \leq B < \infty \quad \text{almost surely}.
\]
Then
\[
\mathbb{E}[\ell(C_{\hat{\lambda}}(X_{n+1}), Y_{n+1})] \leq \alpha.
\]
\end{theorem}

\begin{proof}
Let $\hat{R}_{n+1}(\lambda) = (\ell(C_\lambda(X_1), Y_1) + \ldots + \ell(C_\lambda(X_n), Y_n))/(n+1)$ and define
\[
\hat{\lambda}' = \inf \left\{ \lambda \in \Lambda : \hat{R}_{n+1}(\lambda) \leq \alpha \right\}.
\]

Since $\inf_\lambda \ell(C_\lambda(X_i), Y_i) = \ell(C_{\lambda_{\max}}(X_i), Y_i) \leq \alpha$, $\hat{\lambda}'$ is well-defined almost surely. Since $\ell(C_\lambda(X_i), Y_i) \leq B$, we know
\[
\frac{n}{n+1} \hat{R}_n(\lambda) + \frac{\ell(C_\lambda(X_{n+1}), Y_{n+1})}{n+1} \leq \frac{n}{n+1} \hat{R}_n(\lambda) + \frac{B}{n+1}.
\]

Thus,
\[
\frac{n}{n+1} \hat{R}_n(\hat{\lambda}) + \frac{B}{n+1} \leq \alpha \implies \hat{R}_{n+1}(\hat{\lambda}) \leq \alpha.
\]

This implies $\hat{\lambda}' \leq \hat{\lambda}$ when the LHS holds for some $\lambda \in \Lambda$. When the LHS is above $\alpha$ for all $\lambda \in \Lambda$, by definition, $\hat{\lambda} = \lambda_{\max} \geq \hat{\lambda}'$. Thus, $\hat{\lambda}' \leq \hat{\lambda}$ almost surely. Since $\ell(C_\lambda(X_i), Y_i)$ is non-increasing in $\lambda$,
\[
\mathbb{E}[\ell(C_{\hat{\lambda}}(X_{n+1}), Y_{n+1})] \leq \mathbb{E}[\ell(C_{\hat{\lambda}'}(X_i), Y_i)].
\]

Let $E$ be the multiset of loss functions $\{\ell(C_\lambda(X_1), Y_1), \ldots, \ell(C_\lambda(X_{n+1}), Y_{n+1})\}$. Then $\hat{\lambda}'$ is a function of $E$, or, equivalently, $\hat{\lambda}'$ is a constant conditional on $E$. Additionally, $\ell(C_\lambda(X_{n+1}), Y_{n+1}) \mid E \sim \text{Uniform}(\{\ell(C_\lambda(X_1), Y_1), \ldots, \ell(C_\lambda(X_{n+1}), Y_{n+1})\})$ by exchangeability. These facts, combined with the right-continuity of $\ell(C_\lambda(X_i), Y_i)$, imply
\[
\mathbb{E}[\ell(C_{\hat{\lambda}'}(X_{n+1}), Y_{n+1}) \mid E] = \frac{1}{n+1} \sum_{i=1}^{n+1} \ell(C_{\hat{\lambda}'}(X_i), Y_i) \leq \alpha.
\]

The proof is completed by the law of total expectation and the properties of the loss function.
\end{proof}

% Todo list
% TODO 1: run experiments for 1 2 and 4 within the domain, and run one study cross domains, do multiple rounds, table average and error bar --> taking the average of prediciton and quntile (interval)
%TODO 2: Check CRC implementation (resolve the data leakage issue) 
%TODO 3: section 3.4/3.5 is not detailed enough, write down the penalty term, etc
%TODO 4: complete the appendix
%TODO 5: write about distribution shift between dataset  in the discussion
% TODO 6: we need better models by searching for better parameters

\end{document}